\documentclass[final]{article}

\usepackage[nonatbib,final]{nips_2017}
\usepackage[square,numbers]{natbib}
\usepackage{amsmath,amssymb,amsthm,mathrsfs,bbm}
\usepackage{epsfig,epsf,subfigure,graphicx,graphics}
\usepackage{url, geometry, algorithm, algorithmic, appendix}

\usepackage{xifthen}

\newtheorem{lemma}{Lemma}
\newtheorem{theorem}{Theorem}

\newtheorem{proposition}{Proposition}

\newcommand{\aaaa}{\mathrm{(a)}}
\newcommand{\bbbb}{\mathrm{(b)}}
\newcommand{\cccc}{\mathrm{(c)}}
\newcommand{\dddd}{\mathrm{(d)}}
\newcommand{\eeee}{\mathrm{(e)}}
\newcommand{\ffff}{\mathrm{(f)}}
\newcommand{\gggg}{\mathrm{(g)}}
\newcommand{\hhhh}{\mathrm{(h)}}

\newcommand{\Prob}[1]{\mathbb{P} \left( #1 \right)}
\newcommand{\lp}{\left(}
\newcommand{\rp}{\right)}
\newcommand{\lb}{\left[}
\newcommand{\rb}{\right]}
\newcommand{\lbp}{\left\{}
\newcommand{\rbp}{\right\}}

\newcommand{\wtild}{\widetilde}

\newcommand{\trace}[1]{\textbf{tr}\lp #1\rp}

\DeclareMathOperator*{\argmin}{arg\,min}

\newcommand{\bmat}{\left[\begin{matrix}}
\newcommand{\emat}{\end{matrix}\right]}
\newcommand{\R}{\mathbb R}

\usepackage{color}

\usepackage[utf8]{inputenc} 
\usepackage[T1]{fontenc}    
\usepackage{hyperref}       
\usepackage{url}            
\usepackage{booktabs}       
\usepackage{amsfonts}       
\usepackage{nicefrac}       
\usepackage{microtype}      
\title{Straggler Mitigation in Distributed Optimization Through Data Encoding}

\author{
  Can Karakus \\
  UCLA\\
  Los Angeles, CA \\
  \texttt{karakus@ucla.edu}
  \And
  Yifan Sun \\
  Technicolor Research \\
  Los Altos, CA \\
  \texttt{Yifan.Sun@technicolor.com}
  \AND
  Suhas Diggavi \\
  UCLA \\
  Los Angeles, CA \\
  \texttt{suhasdiggavi@ucla.edu}
  \And
  Wotao Yin \\
  UCLA \\
  Los Angeles, CA \\
  \texttt{wotaoyin@math.ucla.edu} 
}

\begin{document}


\maketitle

\begin{abstract}

Slow running or straggler tasks can significantly reduce computation speed in distributed computation. Recently, coding-theory-inspired approaches have been applied to mitigate the effect of straggling, through embedding redundancy in certain linear computational steps of the optimization algorithm, thus completing the computation without waiting for the stragglers. In this paper, we propose an alternate approach where we embed the redundancy directly in the data itself, and allow the computation to proceed completely oblivious to encoding. We propose several encoding schemes, and demonstrate that popular batch algorithms, such as gradient descent and L-BFGS, applied in a coding-oblivious manner, deterministically achieve sample path linear convergence to an approximate solution of the original problem, using an arbitrarily varying subset of the nodes at each iteration. Moreover, this approximation can be controlled by the amount of redundancy and the number of nodes used in each iteration. We provide experimental results demonstrating the advantage of the approach over uncoded and data replication strategies.
\end{abstract}

\section{Introduction}\label{sec:intro}
Solving large-scale optimization problems 
has become feasible through distributed implementations. However, the efficiency can be significantly hampered by slow processing nodes, network delays or node failures.
In this paper we develop an optimization framework based on encoding the dataset, which mitigates the effect of straggler nodes in the distributed computing system. 
Our approach can be readily adapted to the existing distributed computing infrastructure and software frameworks, since the node computations are oblivious to the data encoding. 

In this paper, we focus on problems of the form
\begin{align}\label{eq:prob1}
\min_{w \in \mathbb{R}^p} f(w) := \frac{1}{2n}\min_{w \in \mathbb{R}^p}\|  Xw - y\|^2,
\end{align}
where $X \in \mathbb{R}^{n \times p}$, $y \in \mathbb{R}^{n \times 1}$ represent the data matrix and vector respectively. The function $f(w)$ is mapped onto a distributed computing setup
depicted in Figure~\ref{fig:dist}, consisting of one central server
and $m$ worker nodes, which collectively store the row-partitioned matrix $X$ and vector $y$. 
We focus on batch, synchronous optimization methods, where the delayed or failed nodes can significantly slow down the overall computation. Note that asynchronous methods are inherently robust to delays caused by stragglers, although their convergence rates can be worse than their synchronous counterparts.
Our approach consists of adding redundancy by encoding the data $X$ and $y$ into $\wtild X = SX$ and $\wtild y = Sy$, respectively, where
$S\in\mathbb{R}^{(\beta n)\times n}$ is an encoding matrix with redundancy factor\ $\beta \ge 1$, and solving the effective problem
\begin{align}\label{eq:prob2}
    \min_{w \in \mathbb{R}^p} \wtild f(w) := \min_{w \in \mathbb{R}^p} \frac{1}{2\beta n}\| S\lp Xw - y\rp\|^2 = \min_{w \in \mathbb{R}^p}  \frac{1}{2\beta n}\| \wtild Xw - \wtild y \|^2
\end{align}
instead.
In doing so, we proceed with the computation in each iteration without waiting for the stragglers, with the idea that the inserted redundancy will compensate for the lost data. The goal is to design the matrix $S$ such that, when the nodes \emph{obliviously} solve the problem \eqref{eq:prob2} without waiting for the slowest $(m-k)$ nodes (where $k$ is a design parameter) the achieved solution approximates the original solution $w^* = \argmin_w f(w)$ sufficiently closely. Since in large-scale machine learning and data analysis tasks one is typically not interested in the exact optimum, but rather a ``sufficiently" good solution that achieves a good generalization error, such an approximation could be acceptable in many scenarios. Note also that the use of such a technique does not preclude the use of other, non-coding straggler-mitigation strategies (\emph{e.g.}, \cite{MTLstragglerJMLR16, WangJoshi_15, AnanthanarayananGhodsi_13} and references therein), which can still be implemented on top of the redundancy embedded in the system, to potentially further improve performance.


Focusing on gradient descent and L-BFGS algorithms, we show that under a spectral condition on $S$, one can achieve an approximation of the solution of \eqref{eq:prob1}, by solving \eqref{eq:prob2}, without waiting for the stragglers. We show that with sufficient redundancy embedded, and with updates from a sufficiently large, yet strict subset of the nodes in each iteration, it is possible to \emph{deterministically} achieve linear convergence to a neighborhood of the solution, as opposed to convergence in expectation (see Fig. \ref{fig:ec2}). Further, one can adjust the approximation guarantee by increasing the redundancy and number of node updates waited for in each iteration. 
Another potential advantage of this strategy is privacy, since the nodes do not have access to raw data itself, but can still perform the optimization task over the jumbled data to achieve an approximate solution.

Although in this paper we focus on quadratic objectives and two specific algorithms, in principle our approach can be generalized to more general, potentially non-smooth objectives and constrained optimization problems, as we discuss in Section~\ref{sec:code} ( adding a regularization term is also a simple generalization).

Our main contributions are as follows. (i) We demonstrate that gradient descent (with constant step size) and L-BFGS (with line search) applied in a coding-oblivious manner on the encoded problem, achieves (universal) sample path linear
convergence to an approximate solution of the
original problem, using only a fraction of the nodes at each iteration.
(ii) We present three classes of coding matrices; namely, equiangular tight frames (ETF), fast transforms, and random matrices, and discuss their properties. 
(iii) We provide experimental results demonstrating the advantage of the approach over uncoded ($S=I$) and data replication strategies, for ridge regression using synthetic data on an AWS cluster, as well as matrix factorization for the Movielens 1-M recommendation task. 

\paragraph{Related work.} Use of data replication to aid with the straggler problem has
been proposed and studied in \cite{WangJoshi_15, AnanthanarayananGhodsi_13}, and references therein. Additionally, use of
coding in distributed computing has been explored in
\cite{LeeLam_16, DuttaCadambe_16}. However, these works exclusively focused on
using coding at the computation level, \emph{i.e.}, certain linear computational steps are performed in a coded manner, and explicit encoding/decoding operations are performed at each step.
Specifically, \cite{LeeLam_16} used MDS-coded distributed matrix
multiplication and \cite{DuttaCadambe_16} focused on breaking up large dot
products into shorter dot products, and perform redundant copies of
the short dot products to provide resilience against stragglers.
\cite{TandonLei_16} considers a gradient descent method on an architecture
where each data sample is replicated across nodes, and designs a
code such that the exact gradient can be recovered as long as fewer
than a certain number of nodes fail. However, in order to recover the exact gradient under any potential set of stragglers, the required
redundancy factor is on the order of the number of straggling nodes, which could mean a large amount of overhead for a large-scale system. In
contrast, we show that one can converge to an approximate solution with
a redundancy factor independent of network size or problem dimensions (\emph{e.g.,} $2$ as in Section \ref{sec:numerical}). 

Our technique is also closely related to randomized linear algebra and sketching techniques \cite{Mahoney_11, DrineasMahoney_11, PilanciWainwright_15}, used for dimensionality reduction of large convex optimization problems. The main difference between this literature and the proposed coding technique is that the former focuses on reducing the problem dimensions to lighten the computational load, whereas coding \emph{increases} the dimensionality of the problem to provide robustness. As a result of the increased dimensions, coding can provide a much closer approximation to the original solution compared to sketching techniques. 
  

\section{Encoded Optimization Framework}\label{sec:encoded}
Figure \ref{fig:dist}  shows a typical data-distributed computational model in large-scale optimization (left), as well as our proposed encoded model (right).
Our computing network consists of $m$ machines, where machine $i$ stores $\lp \wtild X_i, \wtild y_i\rp=\lp S_iX, S_iy\rp$ and $S =\lb S_1^\top \;\; S_2^\top\;\; \dots\;\;S_m^\top\rb^\top$.
The optimization process is oblivious to the encoding, \emph{i.e.}, once the data is stored at the nodes, the optimization algorithm proceeds exactly as if the nodes contained uncoded, raw data $(X,y)$.
\begin{figure}
\centering
\begin{minipage}{.46\textwidth}
  \centering
  \includegraphics[scale=0.5]{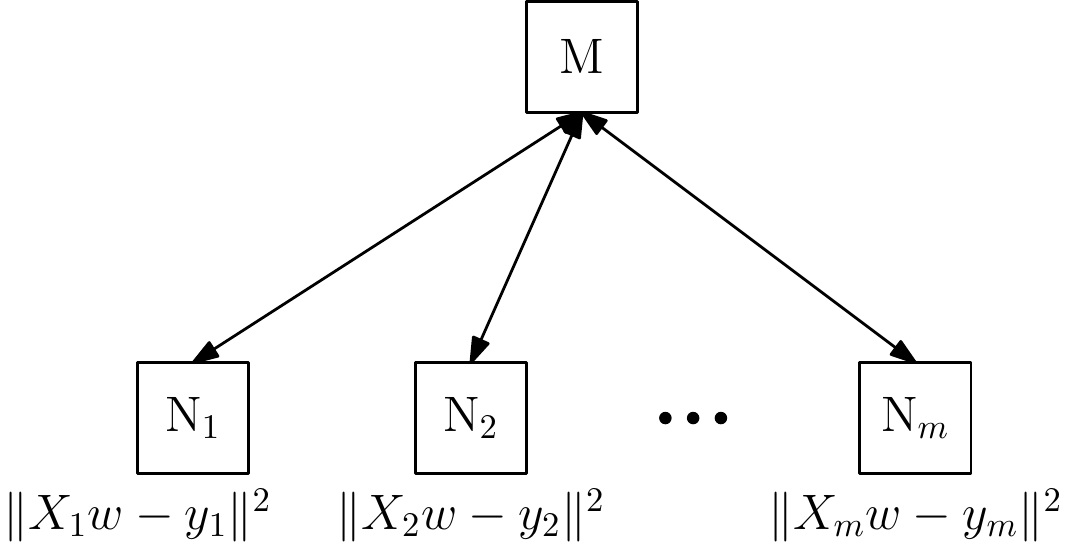}
\end{minipage}\hfill
\begin{minipage}{.46\textwidth}
  \centering
  \includegraphics[scale=0.5]{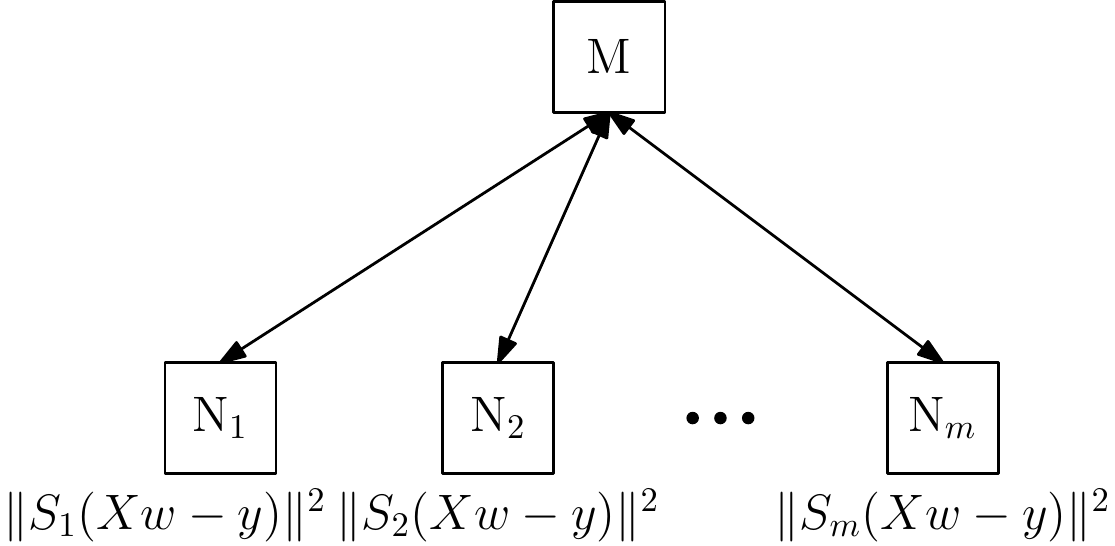}
\end{minipage}
\caption{{\bf Left:} Uncoded distributed optimization with partitioning, where $X$ and $y$ are partitioned as $X = \lb X^\top_1 \; X^\top_2 \; \dots \; X^\top_m\rb^\top$ and $y = \lb y^\top_1 \; y^\top_2 \; \dots \; y^\top_m\rb^\top$.
{\bf Right:} Encoded distributed optimization, where node $i$ stores $\lp S_i X, S_i y\rp$, instead of $\lp X_i, y_i\rp$. The uncoded case corresponds to $S=I$.}
\label{fig:dist}
\end{figure}
In each iteration $t$, the central server broadcasts the current estimate $w_t$, and each worker machine computes and sends to the server the gradient terms corresponding to its own partition $g_i (w_t) := \wtild X_i^\top (\wtild X_i w_t- \wtild y_i)$. 

Note that this framework of distributed optimization is typically communication-bound, where communication over a few slow links constitute a significant portion of the overall computation time. We consider a strategy where at each iteration $t$, the server only uses the gradient updates from the first $k$ nodes to respond in that iteration, thereby preventing such slow links and straggler nodes from stalling the overall computation:
\begin{align*}
    \wtild g_t = 
    \frac{1}{2\beta\eta n}\sum_{i \in A_t} g_i (w_t) 
  = \frac{1}{\beta\eta n} \wtild X^\top_A (\wtild  X_A w_t- \wtild y_A),
\end{align*}
where
$A_t \subseteq [m]$, $|A_t|=k$ are the indices of the first $k$ nodes to respond at iteration $t$,
$\eta:=\frac{k}{m}$ and $\wtild X_A = \lb S_i X\rb_{i \in A_t}$.
(Similarly, $S_A = \lb S_i\rb_{i \in A_t}$.)
Given the gradient approximation, the central server then computes a \emph{descent direction $d_t$} through the history of gradients and parameter estimates.
For the remaining nodes $i\not\in A_t$, the server can either send an interrupt signal, or simply drop their updates upon arrival, depending on the implementation. 
%
%


Next, the central server chooses a step size $\alpha_t$, which can be chosen as constant, decaying, or through exact line search
\footnote{Note that exact line search is not more expensive than backtracking line search for a quadratic loss, since it only requires a single matrix-vector multiplication.} by having the workers compute $\wtild X d_t$ that is needed to compute the step size.
We again assume the central server only hears from the fastest $k$ nodes, denoted by $D_t \subseteq [m]$, where $D_t \neq A_t$ in general, to compute
\begin{align}
    \alpha_t = -\nu\frac{d_t^\top \wtild g_t}{d_t^\top \wtild X^\top_D \wtild X_D d_t} \label{eq:exact_ls},
\end{align}
where $\wtild X_D = \lb S_i X\rb_{i \in D_t}$, and $0<\nu<1$ is a back-off factor of choice.

Our goal is to especially focus on the case $k<m$, and design an encoding matrix $S$ such that, for any sequence of sets $\lbp A_t\rbp$, $\lbp D_t\rbp$, $f(w_t)$ universally converges to a neighborhood of $f(w^*)$. Note that in general, this scheme with $k<m$ is not guaranteed to converge for traditionally batch methods like L-BFGS.
Additionally, although the algorithm only works with the encoded function $\wtild f$,
our goal is to provide a convergence guarantee in terms of the \emph{original} function $f$.

\section{Algorithms and Convergence Analysis}\label{sec:convergence}
Let the smallest and largest eigenvalues of $X^\top X$ be denoted by $\mu>0$ and $M>0$, respectively.

Let $\eta$ with $\frac{1}{\beta} < \eta \leq 1$ be given. In order to prove convergence,
we will consider a family of matrices $\lbp S^{(\beta)}\rbp$ where $\beta$ is the aspect ratio (redundancy factor), such that for any $\epsilon>0$, and any $A \subseteq [m]$ with $\left| A \right| = \eta m$,
\begin{align}
     (1-\epsilon) I \preceq S_A^\top S_A \preceq (1+\epsilon) I \label{eq:rip},
\end{align}
for sufficiently large $\beta \geq 1$, where $S_A =\lb S_i\rb_{i \in A}$ is the submatrix associated with subset $A$ (we drop dependence on $\beta$ for brevity). Note that this is similar to the restricted isometry property (RIP) used in compressed sensing \cite{CandesTao_05}, except that \eqref{eq:rip} is only required for submatrices of the form $S_A$. Although this condition is needed to prove worst-case convergence results, in practice the proposed encoding scheme can work well even when it is not exactly satisfied, as long as the bulk of the eigenvalues of $S_A^\top S_A$ lie within a small interval $\lb 1-\epsilon, 1+\epsilon \rb$. 
We will discuss several specific constructions and their relation to property \eqref{eq:rip} in Section~\ref{sec:code}.


\paragraph{Gradient descent.} We consider gradient descent with constant step size, \emph{i.e.},
\begin{align*}
    w_{t+1} = w_t + \alpha d_t = w_t - \alpha \wtild g_t.
\end{align*} 
The following theorem characterizes the convergence of the encoded problem under this algorithm.
\begin{theorem}\label{th:gd}
Let $f_t = f(w_t)$, where $w_t$ is computed using gradient descent with updates from a set of (fastest) workers $A_t$, with constant step size $\alpha_t \equiv \alpha = \frac{2\zeta}{M(1+\epsilon)}$ for some $0 < \zeta \leq 1$, for all $t$. If $S$ satisfies \eqref{eq:rip} with $\epsilon>0$, then for all sequences of $\{ A_t\}$ with cardinality $\left| A_t \right|=k$,
\begin{align*}
f_t \leq \lp \kappa \gamma_1\rp^t f_0 + \frac{\kappa^2(\kappa-\gamma_1)}{1-\kappa\gamma_1}f\lp w^*\rp,\quad t=1,2,\ldots,
\end{align*}
where $\kappa = \frac{1+\epsilon}{1-\epsilon}$, and $\gamma_1 = \lp 1 - \frac{4\mu\zeta(1-\zeta)}{M\lp 1+\epsilon\rp}\rp$, and $f_0=f(w_0)$ is the initial objective value. 
\end{theorem}
The proof is provided in 
Appendix~B,
which relies on the fact that the solution to the effective ``instantaneous" problem corresponding to the subset $A_t$ lies in the set $\{ w:f(w) \leq \kappa^2 f(w^*) \}$, and therefore each gradient descent step attracts the estimate towards a point in this set, which must eventually converge to this set. Note that in order to guarantee linear convergence, we need $\kappa \gamma_1 <1$, which can be ensured by property \eqref{eq:rip}.

Theorem~\ref{th:gd} shows that gradient descent over the encoded problem, based on updates from only $k<m$ nodes, results in \emph{deterministically} linear convergence to a neighborhood of the true solution $w^*$, for sufficiently large $k$, as opposed to convergence in expectation. 
Note that by property \eqref{eq:rip}, by controlling the redundancy factor $\beta$ and the number of nodes $k$ waited for in each iteration, one can control the approximation guarantee. 
For $k=m$ and $S$ designed properly (see Section~\ref{sec:code}), then $\kappa=1$ and  the optimum value of the original function $f\lp w^*\rp$ is  reached. 

\paragraph{Limited-memory-BFGS.} Although L-BFGS is originally a batch method, requiring updates from all nodes, its stochastic variants have also been proposed recently \cite{MokhtariRibeiro_15, BerahasNocedal_16}. The key modification to ensure convergence is that the Hessian estimate must be computed via  gradient components that are common in two consecutive iterations, \emph{i.e.}, 
from the nodes in $A_t \cap A_{t-1}$. We adapt this technique to our scenario.
For $t>0$, define $u_t := w_t - w_{t-1}$, and
\begin{align*}
r_t &: = \frac{m}{2\beta n\left| A_t \cap A_{t-1}\right|}\sum_{i \in A_t \cap A_{t-1}} \lp g_i(w_t) - g_i(w_{t-1})\rp.
\end{align*}
Then once the gradient terms $\lbp g_t\rbp_{i \in A_t}$ are collected, the descent direction is computed by $d_t = -B_t \wtild g_t$,
where $B_t$ is the inverse Hessian estimate for iteration $t$, which is computed by
\begin{align*}
B_t^{(\ell+1)} = V_{j_{\ell}}^\top B_t^{(\ell)} V_{j_{\ell}} + \rho_{j_{\ell}} u_{j_{\ell}}u_{j_{\ell}}^\top, \;\;\; \rho_k = \frac{1}{r_k^\top u_k}, \;\;\; V_k = I - \rho_k r_k u_k^\top
\end{align*}
with $j_{\ell}= t-\wtild \sigma + \ell$, $B_t^{(0)} = \frac{r_t^\top r_t}{r_t^\top u_t} I$, and $B_t := B_t^{(\wtild \sigma)}$ with $\wtild \sigma := \min\lbp t, \sigma\rbp$, where $\sigma$ is the L-BFGS memory length. Once the descent direction $d_t$ is computed, the step size is determined through exact line search, using \eqref{eq:exact_ls}, with back-off factor $\nu=\frac{1-\epsilon}{1+\epsilon}$, where $\epsilon$ is as in \eqref{eq:rip}.

For our convergence result for L-BFGS, we need another assumption on the matrix $S$, in addition to \eqref{eq:rip}. Defining $\breve S_t = \lb S_i\rb_{i \in A_t \cap A_{t-1}}$ for $t>0$, we assume that for some $\delta>0$, 
\begin{align}
    \delta I \preceq \breve S^\top_t \breve S_t \label{eq:overlap}
\end{align}
for all $t>0$. 
Note that this requires that one should wait for sufficiently many nodes to finish so that the overlap set $A_t \cap A_{t-1}$ has more than a fraction $\frac{1}{\beta}$ of all nodes,
and thus the matrix $\breve S_t$ can be full rank. This is satisfied if $\eta \geq \frac{1}{2} + \frac{1}{2\beta}$ in the worst-case, and under the assumption that node delays are i.i.d., it is satisfied in expectation if $\eta \geq \frac{1}{\sqrt{\beta}}$. However, this condition is only required for a worst-case analysis, and  the algorithm may perform well in practice even when this condition is not satisfied. The following lemma shows the stability of the Hessian estimate.

\begin{lemma}\label{lem:hessian_stability}
If \eqref{eq:overlap} is satisfied, then there exist constants $c_1, c_2 >0$ such that for all $t$, the inverse Hessian estimate $B_t$ satisfies $c_1 I \preceq B_t\preceq c_2 I$.
\end{lemma}
The proof, provided in Appendix A, is based on the well-known trace-determinant method. Using Lemma~\ref{lem:hessian_stability}, we can show the following result.
\begin{theorem}\label{th:lbfgs}
Let $f_t = f(w_t)$, where $w_t$ is computed using L-BFGS as described above, with gradient updates from machines $A_t$, and line search updates from machines $D_t$. If $S$ satisfies \eqref{eq:rip} and \eqref{eq:overlap}, for all sequences of $\{ A_t\}, \{D_t\}$ with $\left| A_t \right|=\left| D_t\right|=k$,
\begin{align*}
f_t \leq \lp \kappa \gamma_2\rp^t f_0 + \frac{\kappa^2(\kappa-\gamma_2)}{1-\kappa\gamma_2}f\lp w^*\rp,
\end{align*}
where $\kappa = \frac{1+\epsilon}{1-\epsilon}$, and $\gamma_2 = \lp 1-\frac{4\mu c_1 c_2}{M\lp c_1+c_2\rp^2}\rp$, and $f_0=f(w_0)$ is the initial objective value. 
\end{theorem}
The proof is provided in 
Appendix~B.
Similar to Theorem~\ref{th:gd}, the proof is based on the observation that the solution of the effective problem at time $t$ lies in a bounded set around the true solution $w^*$. As in gradient descent, coding enables linear convergence deterministically, unlike the stochastic and multi-batch variants of L-BFGS \cite{MokhtariRibeiro_15, BerahasNocedal_16}. 

\paragraph{Generalizations.} Although we focus on quadratic cost functions and two specific algorithms, our approach can potentially be generalized for objectives of the form $\left\| Xw-y\right\|^2 + h(w)$ for a simple convex function $h$, \emph{e.g.}, LASSO; or constrained optimization $\min_{w \in \mathcal{C}} \left\| Xw-y\right\|^2$ (see \cite{KarakusSun_17}); as well as other first-order algorithms used for such problems, \emph{e.g.}, FISTA \cite{BeckTeboulle_09}. In the next section we demonstrate that the codes we consider have desirable properties that readily extend to such scenarios.

\section{Code Design}\label{sec:code}
We consider three classes of coding matrices: tight frames, fast transforms, and random matrices.

\paragraph{Tight frames.} A unit-norm \emph{frame} for $\mathbb{R}^n$ is a set of vectors $F = \lbp \phi_i \rbp_{i=1}^{n\beta}$ with $\|\phi_i\|=1$, where $\beta \geq 1$, such that there exist constants $\xi_1 \geq \xi_2>0$ such that, for any $u \in \mathbb{R}^n$, 
\begin{align*}
    \xi_1 \| u\|^2 \leq \sum_{i=1}^{n\beta} \left| \langle u, \phi_i\rangle\right|^2\leq \xi_2 \| u\|^2.
\end{align*}
The frame is  \emph{tight} if the above satisfied with $\xi_1 = \xi_2$. In this case, it can be shown that the constants are equal to the redundancy factor of the frame, \emph{i.e.}, $\xi_1=\xi_2=\beta$. If we form $S \in \mathbb{R}^{(\beta n) \times n}$ by rows that are a \emph{tight frame}, then we have $S^\top S = \beta I$, which ensures $\| Xw - y\|^2=\tfrac{1}{\beta}\| SXw - Sy\|^2$. 
Then for any solution $\wtild w^*$ to the encoded problem (with $k=m$),  
\begin{align*}
\nabla \wtild f(\wtild w^*) = X^\top S^\top S (X \wtild w^*-y) = \beta (X\wtild w^*-y)^\top X = \beta \nabla f( \wtild w^*).
\end{align*}
Therefore, the solution to the encoded problem satisfies the optimality condition for the original problem as well:
\begin{align*}
     \nabla \wtild f(\wtild w^*)= 0, \quad\Leftrightarrow \quad \nabla f( \wtild w^*)= 0, 
\end{align*}
and if $f$ is also strongly convex, then $\wtild w^* = w^*$ is the unique solution. Note that since the computation is coding-oblivious, this is not true in general for an arbitrary full rank matrix, and this is, in addition to property \eqref{eq:rip}, a desired property of the encoding matrix. In fact, this equivalency extends beyond smooth unconstrained optimization, in that
\begin{align*}
    \left\langle \nabla \wtild f(\wtild w^*),  w - \wtild w^*\right\rangle \ge 0,~ \; \forall w \in \mathcal C\quad \Leftrightarrow \quad\left\langle \nabla f( \wtild w^*),  w - \wtild w^*\right\rangle \ge 0, ~\forall w \in \mathcal C
\end{align*}
for any convex constraint set $\mathcal C$,
as well as
\begin{align*}
    -\nabla \wtild f(\wtild w^*) \in \partial h(\wtild w^*),~ \; \quad \Leftrightarrow \quad -\nabla f(\wtild w^*)\in \partial h(\wtild w^*), 
\end{align*}
for any non-smooth convex objective term $h(x)$, where $\partial h$ is the subdifferential of $h$. 
This means that tight frames can be promising encoding matrix candidates for non-smooth and constrained optimization too. In \cite{KarakusSun_17}, it was shown that when $\lbp A_t \rbp$ is static, equiangular tight frames allow for a close approximation of the solution for constrained problems. 

A tight frame is \emph{equiangular} if $\left| \langle \phi_i, \phi_j\rangle \right|$ is constant across all pairs $(i,j)$ with $i \neq j$.

\begin{proposition}[Welch bound \cite{Welch_74}]\label{prop:welch}
Let $F = \lbp \phi_i\rbp_{i=1}^{n\beta}$ be a tight frame. Then $\omega(F) \geq \sqrt{\frac{\beta-1}{2n\beta-1}}$. Moreover, equality is satisfied if and only if $F$ is an equiangular tight frame.
\end{proposition}
Therefore, an ETF minimizes the correlation between its individual elements, making each submatrix $ S_A^\top S_A$ as close to orthogonal as possible, which is promising in light of property \eqref{eq:rip}. 
We specifically evaluate Paley \cite{Paley_33, GoethalsSeidel_67} and Hadamard ETFs \cite{Szollosi_13} (not to be confused with Hadamard matrix, which is discussed next) in our experiments. We also discuss Steiner ETFs \cite{FickusMixon_12} in Appendix D, which enable efficient implementation.

\begin{figure}
\centering
\begin{minipage}{.46\textwidth}
  \centering
  \includegraphics[scale=0.36]{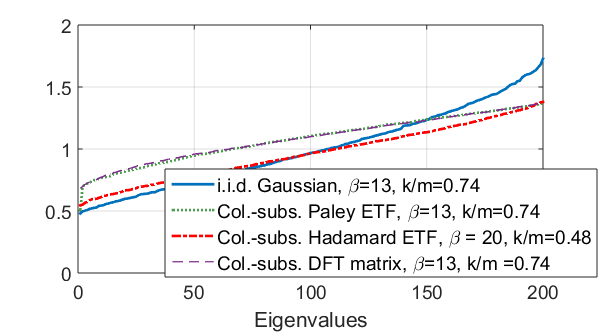}
  \caption{Sample spectrum of $S_A^\top S_A$ for various constructions with high redundancy, and relatively small $k$ (normalized).}
  \label{fig:high_red}
\end{minipage}\hfill
\begin{minipage}{.46\textwidth}
  \centering
  \includegraphics[scale=0.36]{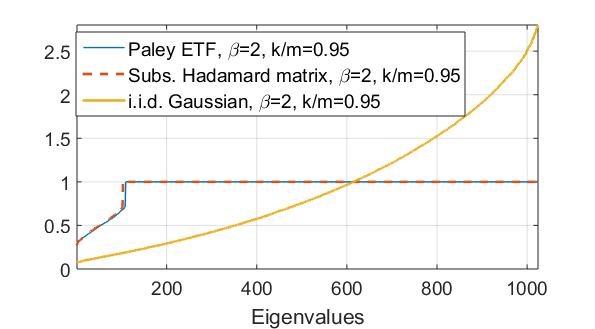}
  \caption{Sample spectrum of $S_A^\top S_A$ for various constructions with low redundancy, and large $k$ (normalized).}
  \label{fig:low_red}
\end{minipage}
\end{figure}

\paragraph{Fast transforms.} Another computationally efficient method for encoding is to use fast transforms: Fast Fourier Transform (FFT), if $S$ is chosen as a subsampled DFT matrix, and the Fast Walsh-Hadamard Transform (FWHT), if $S$ is chosen as a subsampled real Hadamard matrix. In particular, one can insert rows of zeroes at random locations into the data pair $(X,y)$, and then take the FFT or FWHT of each column of the augmented matrix. This is equivalent to a randomized Fourier or Hadamard ensemble, which is known to satisfy the RIP with high probability \cite{CandesTao_06}.

\paragraph{Random matrices.} A natural choice of encoding is using i.i.d. random matrices. Although such random matrices do not have the computational advantages of fast transforms or the optimality-preservation property of tight frames, their eigenvalue behavior can be characterized analytically. In particular, using the existing results on the eigenvalue scaling of large i.i.d. Gaussian matrices \cite{Geman_80, Silverstein_85} and union bound, it can be shown that
\begin{align}
    &\Prob{\max_{A:\left|A\right|=k}\lambda_{\max}\lp \frac{1}{\beta \eta n} S_A^\top S_A\rp > \lp 1 + \sqrt{\frac{1}{\beta \eta}} \rp^2} \to 0, \label{eq:random1}\\& \Prob{\min_{A:\left|A\right|=k}\lambda_{\min}\lp \frac{1}{\beta \eta n} S_A^\top S_A\rp < \lp 1 - \sqrt{\frac{1}{\beta \eta}} \rp^2} \to 0, 
    \label{eq:random2}
\end{align}
as $n \to \infty$, where $\sigma_i$ denotes the $i$th singular value. Hence, for sufficiently large redundancy and problem dimension, i.i.d. random matrices are good candidates for encoding as well. However, for finite $\beta$, even if $k=m$, in general for this encoding scheme the optimum of the original problem is not recovered exactly.



\paragraph{Property \eqref{eq:rip} and redundancy requirements.}
Using the analytical bounds \eqref{eq:random1}--\eqref{eq:random2} on i.i.d. Gaussian matrices, one can see that such matrices satisfy \eqref{eq:rip} with $\epsilon = O\lp \frac{1}{\sqrt{\beta \eta}}\rp$, independent of problem dimensions or number of nodes $m$.
Although we do not have tight eigenvalue bounds for subsampled ETFs, numerical evidence (Figure~\ref{fig:high_red}) suggests that they may satisfy \eqref{eq:rip} with smaller $\epsilon$ than random matrices, 
and thus we believe that the required redundancy in practice is even smaller for ETFs. 

Note that 
our theoretical results focus on the extreme eigenvalues due to a worst-case analysis; in practice, 
most of the energy of the gradient will be on the eigen-space associated with the bulk of the eigenvalues, which the following proposition suggests can be mostly 1 (also see Figure~\ref{fig:low_red}), which means even if \eqref{eq:rip} is not satisfied, the gradient (and the solution) can be approximated closely for a modest redundancy, such as $\beta=2$. The following result is a consequence of the Cauchy interlacing theorem, and the definition of tight frames.
\begin{proposition}\label{prop:etf_bulk}
If the rows of $S$ are chosen to form an ETF with redundancy $\beta$, then for $\eta \geq 1- \frac{1}{\beta}$, $\frac{1}{\beta}S_A^\top S_A$ has $n(1-\beta\eta)$ eigenvalues equal to 1.
\end{proposition}


\section{Numerical Results}\label{sec:numerical}
\begin{figure}
\centering
  \includegraphics[scale=0.27]{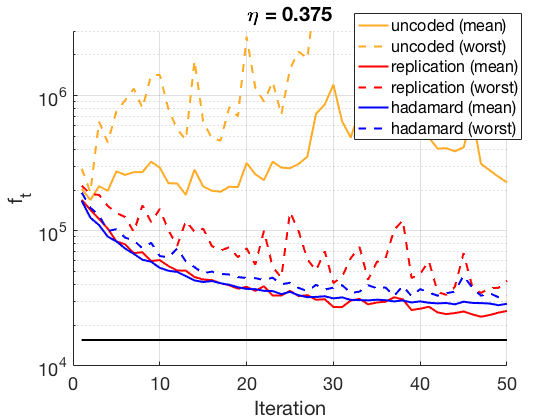}
  \includegraphics[scale=0.27]{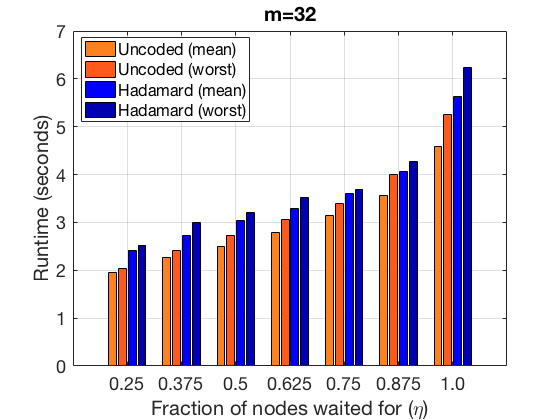}
  \caption{{\bf Left:} Sample evolution of uncoded, replication, and Hadamard (FWHT)-coded cases, for $k=12$, $m=32$. {\bf Right:} Runtimes of the schemes for different values of $\eta$, for the same number of iterations for each scheme. Note that this essentially captures the delay profile of the network, and does not reflect the relative convergence rates of different methods.}
  \label{fig:ec2}
\end{figure}

\paragraph{Ridge regression with synthetic data on AWS EC2 cluster.} We generate the elements of matrix $X$ i.i.d. $\sim N(0,1)$, the elements of $y$ i.i.d. $\sim N(0,p)$, for dimensions $(n,p)=(4096, 6000)$, and solve the problem $\min_w \frac{1}{2\beta n}\left\| \wtild Xw - \wtild y\right\|^2 + \frac{\lambda}{2} \| w\|^2$,
for regularization parameter $\lambda=0.05$. We evaluate column-subsampled Hadamard matrix with redundancy $\beta=2$ (encoded using FWHT for fast encoding), data replication with $\beta = 2$, and uncoded schemes. We implement distributed L-BFGS as described in Section~\ref{sec:convergence} on an Amazon EC2 cluster using the \texttt{mpi4py} Python package, over $m=32$ \texttt{m1.small} worker node instances, and a single \texttt{c3.8xlarge} central server instance. We assume the central server encodes and sends the data variables to the worker nodes (see Appendix D for a discussion of how to implement this more efficiently).

Figure~\ref{fig:ec2} shows the result of our experiments, which are aggregated over 20 trials. As baselines, we consider the uncoded scheme, as well as a replication scheme, where each uncoded partition is replicated $\beta=2$ times across nodes, and the server uses the faster copy in each iteration. It can be seen from the right figure that one can speed up computation by reducing $\eta$ from 1 to, for instance, 0.375, resulting in more than $40\%$ reduction in the runtime. Note that in this case, uncoded L-BFGS fails to converge, whereas the Hadamard-coded case stably converges. We also observe that the data replication scheme converges on average, but in the worst case, the convergence is much less smooth, since the performance may deteriorate if both copies of a partition are delayed. 




\paragraph{Matrix factorization on Movielens 1-M dataset.} We next apply matrix factorization on the MovieLens-1M dataset \cite{RiedlKonstan_98} for the movie recommendation task. We are given $R$, a sparse matrix of movie ratings 1--5, of dimension $\# users \times \# movies$, where $R_{ij}$ is specified if user $i$ has rated movie $j$. We withhold randomly 20\% of these ratings to form an 80/20 train/test split. 
The goal is to recover user vectors $x_i\in \R^p$ and movie vectors $y_i\in \R^p$ (where $p$ is the embedding dimension) such that $R_{ij} \approx  x_i^Ty_j + u_i + v_j +\mu$, where $u_i$, $v_j$, and $\mu$ are user, movie, and global biases, respectively. The optimization problem is given by
\begin{equation}
\min_{x_i, y_j, u_i, v_j} \sum_{i,j \text{: observed}}(R_{ij} - u_i - v_j - x_i^Ty_j-\mu)^2 + \lambda \left(\sum_i \|x_i\|_2^2 +\|u\|_2^2 + \sum_j\|y_j\|_2^2 + \|v\|_2^2\right).
\label{eq:movielens_full}
\end{equation}
We choose $\mu = 3$, $p = 15$, and $\lambda = 10$, which achieves a test RMSE 0.861, close to the current best test RMSE on this dataset using matrix factorization\footnote{\texttt{http://www.mymedialite.net/examples/datasets.html}}. 

\begin{figure}
\centering
\begin{minipage}{.64\textwidth}
  \centering
    \includegraphics[width=1.55in]{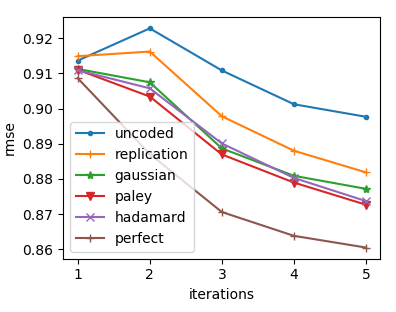}
    \includegraphics[width=1.55in]{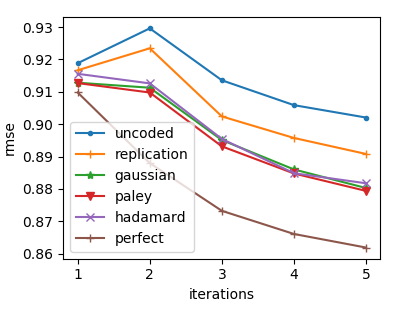}
    \includegraphics[width=1.55in]{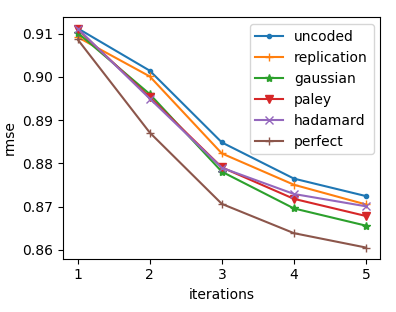}
    \includegraphics[width=1.55in]{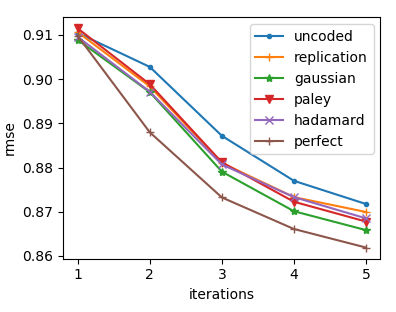}
  \caption{Test RMSE for $m = 8$ (left) and $m = 24$ (right) nodes, where the server waits for $k = m / 8$ (top) and $k = m/2$ (bottom) responses. ``Perfect" refers to the case where $k=m$.}
  \label{fig:movielens_perf}
\end{minipage}\hfill
\begin{minipage}{.33\textwidth}
\vfill
  \centering
    \includegraphics[width=1.55in]{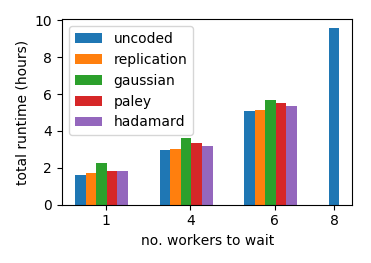}
    \includegraphics[width=1.55in]{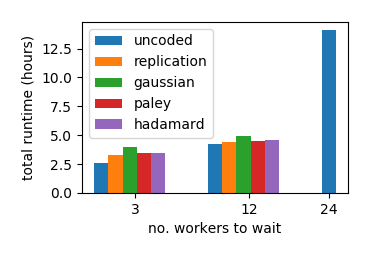}
    \caption{Total runtime with $m=8$ and $m=24$ nodes for different values of $k$, under fixed 100 iterations for each scheme.}
  \label{fig:movielens_time}
\end{minipage}
\end{figure}
Problem \eqref{eq:movielens_full} is often solved using alternating minimization, minimizing first over all $\lp x_i, u_i\rp$, and then all $\lp y_j, v_j\rp$, in repetition. Each such step further decomposes by row and column, made smaller by the sparsity of $R$. To solve for $\lp x_i, u_i\rp$, we first extract $I_i = \{j \mid r_{ij} \text{ is observed}\}$, and solve the resulting sequence of regularized least squares problems in the variables $w_i = [x_i^\top, u_i]^\top$ 
distributedly using coded L-BFGS; and repeat for $w = [y_j^\top, v_j]^\top$, for all $j$. As in the first experiment, distributed coded L-BFGS is solved by having the master node encoding the data locally, and distributing the encoded data to the worker nodes (Appendix D discusses how to implement this step more efficiently). The overhead associated with this initial step is included in the overall runtime in Figure~\ref{fig:movielens_time}.

The Movielens experiment is run on a single 32-core machine
with 256 GB RAM. 
In order to simulate network latency, an artificial delay of $\Delta \sim \text{exp}(\text{10 ms})$ is imposed each time the worker completes a task. Small problem instances ($n< 500$) are solved locally at the central server, using the built-in function \texttt{numpy.linalg.solve}.
Additionally, parallelization is only done for the ridge regression instances, in order to isolate speedup gains in the L-BFGS distribution.
To reduce overhead, we create a bank of encoding matrices $\lbp S_n\rbp$ for Paley ETF and Hadamard ETF, for $n=100, 200, \hdots, 3500$, and then given a problem instance, subsample the columns of the appropriate matrix $S_n$ to match the dimensions.
Overall, we observe that encoding overhead is amortized by the speed-up of the distributed optimization.

Figure \ref{fig:movielens_perf} gives the final performance of our distributed L-BFGS for various encoding schemes, for each of the 5 epochs, which shows that coded schemes are most robust for small $k$. 
A full table of results is given in 
Appendix~C.

\subsubsection*{Acknowledgments}

This work was supported in part by NSF grants 1314937 and 1423271.


\medskip
\small

\bibliography{Ref}

\begin{appendices}
\section{Lemmas}\label{sec:lemmas}
In the proofs, we will ignore the normalization constants on the objective functions for brevity. Let $f^A_t := \| \wtild X_A w_t - \wtild y_A\|^2$, and $f_A(w) := \| \wtild X_A w - \wtild y_A\|^2$ (we set $A \equiv A_t$). Let $\wtild w_t^*$ denote the solution to the effective ``instantaneous" problem at iteration $t$, \emph{i.e.}, $\wtild w_t^* = \argmin_w \| \wtild X_A w-\wtild y_A \|^2$.

Stronger versions of the following lemma has been proved in \cite{KarakusSun_17, PilanciWainwright_15}, but we include a weakened version of this result here for completeness.
\begin{lemma}\label{lem:solution_ball}
 For any $t$ and $\{ A_t\}$,
\begin{align*}
    f(\wtild w_t^*) \leq \kappa^2 f(w^*)
\end{align*}
\end{lemma}
\begin{proof}
Define $e = \wtild w_t^* - w^*$ and note that
\begin{align*}
\| X\wtild w_t^*-y\| = \| X w^*-y + Xe\| \leq \| X w^*-y\| + \| Xe\| 
\end{align*}
by triangle inequality, which implies
\begin{align}
f(\wtild w_t^*) = \| X\wtild w_t^*-y\|^2 \leq \lp 1 + \frac{\| Xe\|}{\| X w^*-y\|}\rp^2 \| X w^*-y\|^2 = \lp 1 + \frac{\| Xe\|}{\| X w^*-y\|}\rp^2 f(w^*) \label{eq:sol_ball_main}.
\end{align}
Now, for any $c>0$, consider
\begin{align*}
\| Xe \|^2 &\leq \frac{\| S_A Xe\|^2}{\lambda_{\min}} \overset{\aaaa}{\leq} -2\frac{e^\top X^\top  S_A^\top S_A (Xw^*-y)}{\lambda_{\min}} \\
&= -2\frac{e^\top X^\top \lp S_A^\top  S_A -cI\rp (Xw^*-y)}{\lambda_{\min}} - \frac{2c}{\lambda_{\min}} e^\top X^\top  (Xw^*-y) \\
&\overset{\bbbb}{=} -2\frac{e^\top X^\top \lp S_A^\top S_A -cI\rp (Xw^*-y)}{\lambda_{\min}} \\
&\overset{\cccc}{\leq} 2\frac{ \left\| e^\top X^\top \lp cI -  S_A^\top S_A \rp \right\|}{\lambda_{\min}} \|Xw^*-y \| \\
&\overset{\dddd}{\leq} 2\frac{ \left\| cI - S_A^\top S_A \right\|}{\lambda_{\min}} \|Xw^*-y \| \|Xe \|,
\end{align*}
where (a) follows by expanding and re-arranging $\left\| \wtild X_A\wtild w_t^* - \wtild y_A\right\|^2 \leq \left\| \wtild X_A w^* - \wtild y_A\right\|^2$, which is since $\wtild w_t^*$ is the minimizer of this function; (b) follows by the fact that $\nabla f(w^*) = X^\top (Xw^*-y)=0$ by optimality of $w^*$ for $f$; (c) follows by Cauchy-Schwarz inequality; and (d) follows by the definition of matrix norm.

Since this is true for any $c>0$, we choose $c=\frac{\lambda_{\max} + \lambda_{\min}}{2}$, which gives
\begin{align*}
\frac{\| Xe \|}{\| X\wtild w_t^*-y\|} \leq \frac{\lambda_{\max} - \lambda_{\min}}{\lambda_{\min}} = \kappa - 1.
\end{align*}
Plugging this back in \eqref{eq:sol_ball_main}, we get $f(\wtild w^*) \leq \kappa^2 f(w^*)$, which completes the proof.
\end{proof}

\begin{lemma}\label{lem:final_arg}
    If
    \begin{align*}
        \wtild f^A_{t+1} - \wtild f^A\lp \bar w\rp \leq \gamma \lp \wtild f^A_{t} - \wtild f^A\lp \bar w\rp\rp
    \end{align*}
    for all $t>0$, for some $\bar w \in \mathbb{R}^p$, and for some $0<\gamma<1$, then
    \begin{align*}
        f_t \leq \lp \kappa\gamma\rp^t f_0 + \frac{\kappa-\gamma}{1-\kappa\gamma}f\lp \bar w\rp.
    \end{align*}
\end{lemma}
\begin{proof}
    Since for any $w$,
\begin{align*}
\lambda_{\min} f(w) = \lambda_{\min} \| Xw-y\|^2 \leq \lp Xw-y\rp^\top  S_A^\top \wtild S_A \lp Xw-y\rp = \wtild f^A(w),
\end{align*}
and similarly $\wtild f^A(w) \leq \lambda_{\max} f(w)$, we have
\begin{align*}
\lambda_{\min} f_{t+1} - \lambda_{\max} f\lp \bar w\rp \leq \gamma \lp \lambda_{\max} f_t - \lambda_{\min} f\lp \bar w\rp\rp,
\end{align*}
which can be re-arranged into the linear recursive inequality
\begin{align*}
f_{t+1} \leq \kappa\gamma f_t + (\kappa-\gamma) f\lp \bar w\rp,
\end{align*}
where $\kappa=\frac{\lambda_{\max}}{\lambda_{\min}}$. By considering such inequalities for $0\leq \tau \leq t$, multiplying each by $\lp \kappa\gamma\rp^{t-\tau}$ and summing, we get
\begin{align*}
f_t &\leq \lp \kappa \gamma\rp^t f_0 + (\kappa-\gamma)f\lp \bar w\rp\sum_{\tau=0}^{t-1} \lp \kappa\gamma\rp^{\tau}\\
&\leq \lp \kappa \gamma\rp^t f_0 + \frac{\kappa-\gamma}{1-\kappa\gamma}f\lp \bar w\rp.
\end{align*}
\end{proof}

\begin{lemma}\label{lem:strong_conv}
$\wtild f^A(w)$ is $\lambda_{\min}\mu$-strongly convex.
\end{lemma}
\begin{proof}
It is sufficient to show that the minimum eigenvalue of $\wtild X_A^\top \wtild X_A$ is bounded away from zero. This can easily be shown by the fact that
\begin{align*}
u^\top\wtild X_A^\top \wtild X_Au = u^\top X^\top S_A^\top S_A X u \geq \lambda_{\min} \| Xu\|^2 \geq \lambda_{\min}\mu \|u\|^2,
\end{align*}
for any unit vector $u$.
\end{proof}

\begin{lemma}\label{lem:rotation_bound}
Let $M \in \mathbb{R}^{p \times p}$ be a positive definite matrix, with the condition number (ratio of maximum eigenvalue to the minimum eigenvalue) given by $\kappa$. Then, for any unit vector $u$,
\begin{align*}
\frac{u^\top Mu}{\| Mu\|} \geq \frac{2\sqrt{\kappa}}{\kappa+1}.
\end{align*}
\end{lemma}
\begin{proof}
Let $\mathcal{W}$ be the subspace spanned by $\lbp u, Mu \rbp$, and let $W \in \mathbb{R}^{p \times 2}$ be a matrix whose columns form an orthonormal basis for $\mathcal{W}$. Then $u$ and $Mu$ can be represented as
\begin{align*}
u &= Wr_1 \\
Mu &= Wr_2
\end{align*}
for some $r_1,r_2 \in \mathbb{R}^2$, which implies
\begin{align*}
r_2 = W^\top M W r_1 = Q^\top D Q r_1,
\end{align*}
where since $W^\top M W$ is still a positive definite matrix it has the $2\times 2$ eigen-decomposition $Q^\top D Q$. Defining $q_i = Qr_i$ for $i=1,2$, note that the quantity we are interested in can be equivalently represented as
\begin{align*}
\frac{u^\top Mu}{\| Mu\|} = \frac{r_1^\top r_2}{\| r_2\|} = \frac{q_1^\top q_2}{\| q_2\|},
\end{align*}
where $q_2 = Dq_1$. Further note that for any unit vector $v$,
\begin{align*}
v^\top D v = v^\top Q^\top D Qv = v^\top W^\top M W v, 
\end{align*}
and since $\| Wv\|=1$, the condition number of $D$ (the ratio of the two non-zero elements of $D$) cannot be larger than that of $M$, which is $\kappa$ (since otherwise once could find unit vectors $u_1 = Wv_1$ and $u_2=Wv_2$ such that $\frac{u_1^\top M u_1}{u_2^\top M u_2}> \kappa$, which is a contradiction). Representing $q_1 = \lb \cos\phi \;\; \sin\phi\rb^\top$ for some angle $\phi$, $\frac{q_2}{\| q_2 \|}$ can then be written as $q_2 = \lb \frac{d_1}{\sqrt{d_1^2 + d_2^2}} \cos \phi \;\; \frac{d_2}{\sqrt{d_1^2 + d_2^2}} \sin \phi\rb^\top$. Note that minimizing the inner product $\frac{q_1^\top q_2}{\| q_2\|}$ is equivalent to maximizing the function
\begin{align*}
\tan\lp \tan^{-1}\lp \frac{d_2}{d_1} \tan \phi\rp- \phi\rp = \frac{\lp\frac{d_2}{d_1} - 1\rp \tan \phi}{1 + \frac{d_2}{d_1} \tan^2 \phi}
\end{align*}
over $\phi$. By setting the derivative to zero, we find that the maximizing $\phi$ is given by $\cos^{-1} \frac{d_2-d_1}{d_2+d_1}$. Therefore
\begin{align*}
\frac{u^\top Mu}{\| Mu\|} = \frac{q_1^\top q_2}{\| q_2\|} = \frac{d_1}{\sqrt{d_1^2 + d_2^2}} \cos^2 \phi + \frac{d_2}{\sqrt{d_1^2 + d_2^2}} \sin^2 \phi \geq \frac{2\sqrt{\frac{d_2}{d_1}}}{1 + \frac{d_2}{d_1}} \geq \frac{2\sqrt{\kappa}}{1 + \kappa},
\end{align*}
which is the desired result.

\end{proof}

\begin{proof}
[Proof of Lemma~1]
First note that
\begin{align}
r_t^\top u_t &= \lp X^\top \breve S_t^\top \breve S_t \lb (Xw_t - y) - (Xw_{t-1} - y)\rb\rp^\top \lp w_t - w_{t-1}\rp \notag\\
& = \lp w_t - w_{t-1}\rp^\top X^\top \breve S_t^\top \breve S_t X\lp w - w_{t-1}\rp \notag\\
&\geq \epsilon \mu \|u_t\|^2, \label{eq:bd1}
\end{align}
by 
(5)
Also consider
\begin{align*}
\frac{\|r_t \|^2}{r_t^\top u_t} = \frac{\lp w_t - w_{t-1}\rp^\top \lp X^\top \breve S_t^\top \breve S_t X \rp^2\lp w_t - w_{t-1}\rp }{\lp w_t - w_{t-1}\rp^\top X^\top \breve S_t^\top \breve S_t X\lp w_t - w_{t-1}\rp},
\end{align*}
which implies
\begin{align*}
\epsilon \mu \leq \frac{\|r_t \|^2}{r_t^\top u_t}  \leq (1+\epsilon) M,
\end{align*}
again by 
(4).
Now, setting $j_\ell = t - \wtild \sigma + \ell$, consider the trace
\begin{align*}
\trace{B_t^{(\ell+1)}} &= \trace{ B_t^{(\ell)}} - \trace{\frac{B_t^{(\ell)}u_{j_\ell} u_{j_\ell}^\top B_t^{(\ell)} }{u_{j_\ell}^\top B_t^{(\ell)} u_{j_\ell}}} + \trace{\frac{r_{j_\ell} r_{j_\ell}^\top}{r_{j_\ell}^\top u_{j_\ell}} } \\
&\leq \trace{ B_t^{(\ell)}} +  \trace{\frac{r_{j_\ell} r_{j_\ell}^\top}{r_{j_\ell}^\top u_{j_\ell}} } \\
&=  \trace{ B_t^{(\ell)}} +\frac{\|r_{j_\ell} \|^2}{r_{j_\ell}^\top u_{j_\ell}} \\
&\leq  \trace{ B_t^{(\ell)}} +(1+\epsilon) M,
\end{align*}
which implies $\trace{B_t} \leq (1+\epsilon) M \lp \wtild \sigma + d\rp$. It can also be shown (similar to \cite{BerahasNocedal_16}) that
\begin{align*}
\det\lp B_t^{(\ell+1)}  \rp &=  \det\lp B_t^{(\ell)}\rp \cdot\frac{r^\top_{j_\ell} u_{j_\ell}}{u_{j_\ell}^\top B_t^{(\ell)} u_{j_\ell}} \\
&= \det\lp B_t^{(\ell)}\rp \cdot \frac{r^\top_{j_\ell} u_{j_\ell}}{\| u_{j_\ell}\|^2} \cdot \frac{\| u_{j_\ell}\|^2}{u_{j_\ell}^\top B_t^{(\ell)} u_{j_\ell}} \\
&\geq  \det\lp B_t^{(\ell)}\rp \frac{\epsilon \mu}{(1+\epsilon) M \lp \wtild \sigma + d\rp},
\end{align*}
which implies $\det\lp B_t \rp \geq \det\lp B_t^{(0)}\rp\lp\frac{\epsilon \mu}{(1+\epsilon) M \lp \wtild \sigma + d\rp}\rp^{\wtild \sigma}$. Since $B_t \geq 0$, its trace is bounded above, and its determinant is bounded away from zero, there must exist $0 < c_1 \leq c_2$ such that
\begin{align*}
c_1 I \preceq B_t \preceq c_2 I.
\end{align*}
\end{proof}

\section{Proofs of 
Theorem~1
and 
Theorem~2
}\label{sec:proof}

Throughout the section, we will consider a particular iteration $t$, and denote
\begin{align*}
\lambda_{\min} &:= \min\lbp \lambda_{\min}\lp S_A^\top S_A\rp, \lambda_{\min}\lp  S_D^\top S_D\rp\rbp \\
\lambda_{\max} &:= \max\lbp \lambda_{\max}\lp  S_A^\top  S_A\rp, \lambda_{\max}\lp  S_D^\top S_D\rp\rbp,
\end{align*}
where $\lambda_{\min}(\cdot)$ and $\lambda_{\max}(\cdot)$ denote the minimum and maximum eigenvalues of a matrix.

We will also denote with $\wtild w_t^*$ the solution to the effective ``instantaneous" problem at iteration $t$, \emph{i.e.}, $\wtild w_t^* = \argmin_w \left\| \wtild X_A w-\wtild y_A \right\|^2$, where $A \equiv A_t$, and we ignore the normalization constants on the objective functions for. Finally, we define $\wtild f^A(w) := \| \wtild X_A w - \wtild y_A\|^2$, and $\wtild f^A_t := \| \wtild X_A w_t - \wtild y_A\|^2$.

\subsection{Proof of Theorem~1}

Using convexity, and the choices that $d_t =-\wtild g_t$ and $\alpha_t = \alpha$, we have
\begin{align*}
&\wtild f^A\lp w_t - \alpha_t d_t\rp - \wtild f^A(w_t) \leq \alpha_t \wtild g_t^\top d_t + \frac{1}{2} \alpha_t^2 d_t^\top X^\top S_A^\top S_A X d_t \\
&= -\alpha \|\wtild g_t \|^2 + \frac{1}{2} \alpha^2 \wtild g_t^\top X^\top S_A^\top S_A X \wtild g_t \overset{\aaaa}{\leq} -\alpha \|\wtild g_t \|^2 + \frac{\lambda_{\max}}{2} \alpha^2 \wtild g_t^\top X^\top X \wtild g_t \\
&\overset{\bbbb}{\leq} -\alpha \lp 1 - \frac{\lambda_{\max} M}{2} \alpha\rp \|\wtild g_t \|^2 = -\frac{2\zeta \lp 1 - \zeta\rp}{M\lambda_{\max}}  \|\wtild g_t \|^2 \\
&\overset{\cccc}{\leq} -\frac{4\mu \zeta \lp 1 - \zeta\rp}{M\lambda_{\max}} \lp \wtild f^A\lp w_t\rp - \wtild f^A\lp \wtild w_t^*\rp\rp,
\end{align*}
where (a) follows by the fact that $S_A^\top S_A \preceq \lambda_{\max}I$; (b) follows since $X^\top X \preceq MI$; and (c) follows by strong convexity.
Re-arranging this inequality, and using the definition of $\gamma_1$, we get
\begin{align*}
    \wtild f^A_{t+1} - \wtild f^A\lp \wtild w_t^*\rp \leq \gamma_1 \lp \wtild f^A_t - \wtild f^A\lp \wtild w_t^*\rp\rp.
\end{align*}
Then, Lemma~\ref{lem:final_arg} with $\bar w = \wtild w_t$ implies
\begin{align*}
    f_t \leq \lp \kappa\gamma_1\rp^t f_0 + \frac{\kappa-\gamma_1}{1-\kappa\gamma_1}f\lp \wtild w_t^*\rp.
\end{align*}
Finally, Lemma~\ref{lem:solution_ball} implies $f(\wtild w_t^*) \leq \kappa^2 f(w^*)$, which concludes the proof.


\subsection{Proof of Theorem~2}

Using convexity and the closed-form expression for the step size, we have
\begin{align*}
&\wtild f^A\lp w_t + \alpha_t d_t\rp - \wtild f^A(w_t) \leq \alpha_t \wtild g_t^\top d_t + \frac{1}{2} \alpha_t^2 d_t^\top X^\top  S_A^\top  S_A X d_t \\
&= - \frac{ \nu\lp \wtild g_t^\top d_t \rp^2}{d_t^\top X^\top  S_D^\top  S_D X d_t} + \frac{1}{2} \frac{\nu^2\lp \wtild g_t^\top d_t\rp^2}{d_t^\top X^\top S_D^\top  S_D X d_t}\cdot \frac{d_t^\top X^\top S_A^\top S_A X d_t}{d_t^\top X^\top S_D^\top S_D X d_t} \\
&= \lp \frac{d_t^\top X^\top \lp \nu^2 S_A^\top S_A - 2\nu S_D^\top S_D\rp  X d_t}{2\lp d_t^\top X^\top S_D^\top S_D X d_t\rp^2}\rp\lp d_t^\top \wtild g_t\rp^2 \overset{\aaaa}{=} -\nu\lp \frac{z^\top \lp  S_D^\top S_D - \frac{\nu}{2} S_A^\top S_A\rp z}{\lp z^\top S_D^\top S_D z\rp^2}\rp \frac{\lp d_t^\top \wtild g_t\rp^2}{\| Xd_t\|^2} \\
&\overset{\bbbb}{\leq} -\nu\lp \frac{\lambda_{\min}-\frac{\nu}{2}\lambda_{\max}}{\lambda_{\min}^2}\rp \frac{\lp d_t^\top \wtild g_t\rp^2}{\| Xd_t\|^2} \overset{\cccc}{\leq} -\frac{\nu}{M}\lp \frac{\lambda_{\min}-\frac{\nu}{2}\lambda_{\max}}{\lambda_{\min}^2}\rp \frac{\lp d_t^\top \wtild g_t\rp^2}{\| d_t\|^2} \\ &\overset{\dddd}{=} -\frac{\nu}{M}\lp \frac{\lambda_{\min}-\frac{\nu}{2}\lambda_{\max}}{\lambda_{\min}^2}\rp \frac{\lp \wtild g_t^\top B_t \wtild g_t\rp^2}{\| B_t \wtild g_t\|^2} \\
&\overset{\eeee}{\leq} -\frac{4\nu}{M}\lp \frac{\lambda_{\min}-\frac{\nu}{2}\lambda_{\max}}{\lambda_{\min}^2}\rp \frac{c_1 c_2}{\lp c_1 + c_2\rp^2}\| \wtild g_t\|^2 \overset{\ffff}{\leq} -\frac{8\mu\nu}{M}\lp \frac{\lambda_{\min}-\frac{\nu}{2}\lambda_{\max}}{\lambda_{\min}^2}\rp\frac{c_1 c_2}{\lp c_1 + c_2\rp^2} \lp \wtild f\lp w_t\rp - \wtild f\lp \wtild w_t^*\rp\rp\\
&\overset{\gggg}{=} -\frac{4\mu c_1 c_2}{M\lp c_1 + c_2\rp^2}\lp \wtild f\lp w_t\rp - \wtild f\lp \wtild w_t^*\rp\rp \overset{\hhhh}{=} -\lp 1-\gamma_2\rp \lp \wtild f^A\lp w_t\rp - \wtild f^A\lp \wtild w_t^*\rp\rp.
\end{align*}
where (a) follows by defining $z = \frac{Xd_t}{\| Xd_t\|}$; (b) follows by the fact that the term in parenthesis is an increasing function of the quadratic form $z^\top \breve S^\top_t \breve S_t z$ and by Assumption 1; (c) follows by the assumption that $X^\top X \preceq MI$; (d) follows by the definition of $d_t$; (e) follows by 
Lemmas~\ref{lem:rotation_bound} 
and 
1;
(f) follows by strong convexity of $\wtild f$ (by Lemma~\ref{lem:strong_conv}), which implies $\| \wtild g_t\|^2 \geq 2\mu \lp \wtild f\lp \theta_t\rp - \wtild f\lp \wtild w_t^*\rp\rp$; (g) follows by choosing $\nu = \frac{\lambda_{\min}}{\lambda_{\max}}$; and (h) follows using the definition of $\gamma_2$.

Re-arranging the inequality, we obtain
\begin{align*}
    \wtild f^A_{t+1} - \wtild f^A\lp \wtild w_t^*\rp \leq \gamma_2 \lp \wtild f^A_t - \wtild f^A\lp \wtild w_t^*\rp\rp,
\end{align*}
and hence applying first Lemma~\ref{lem:final_arg} with $\bar w = \wtild w_t$, and then Lemma~\ref{lem:solution_ball}, we get the desired result.

\section{Full results of Movielens 1-M experiment}\label{sec:a-movielens-table}
\label{a-movielens-table}


Tables \ref{t-movielens8} and \ref{t-movielens24} give the test and train RMSE for the Movielens 1-M recommendation task, with a random 80/20 train/test split.

\begin{table}[htbp]
\centering
\begin{tabular}{|c|c|c|c|c|c|ccc}
\hline
& uncoded & replication & gaussian & paley & hadamard\\
\hline
&\multicolumn{5}{|c|}{$m=8$, $k=1$}\\
\hline
train RMSE &0.804 & 0.783  & 0.781 & \textbf{0.775} & 0.779 \\
test RMSE &0.898 & 0.889  & 0.877 & \textbf{0.873} & 0.874 \\
runtime &1.60 & 1.76 & 2.24 & 1.82 & 1.82  \\
\hline
&\multicolumn{5}{|c|}{$m=8$, $k=4$}\\
\hline
train RMSE &0.770 & 0.766 &  0.765 & \textbf{0.763} & 0.765 \\
test RMSE &0.872 & 0.872 & \textbf{0.866} & 0.868 & 0.870 \\
runtime &2.96 & 3.13 & 3.64 & 3.34 & 3.18  \\
\hline
&\multicolumn{5}{|c|}{$m=8$, $k=6$}\\
\hline
train RMSE &0.762 & 0.760  & 0.762 & \textbf{0.758} & 0.760 \\
test RMSE &0.866 & 0.871  & 0.864 & \textbf{0.860} & 0.864 \\
runtime &5.11 & 4.59 & 5.70 & 5.50 & 5.33 \\
\hline
\end{tabular}
\caption{Full results for Movielens 1-M, distributed over $m=8$ nodes total. Runtime is in hours. An uncoded scheme running full batch L-BFGS has a train/test RMSE of 0.756 / 0.861, and a runtime of 9.58 hours.}
\label{t-movielens8}
\end{table}

\begin{table}[htbp]
\centering
\begin{tabular}{|c|c|c|c|c|c|ccc}
\hline
& uncoded & replication  & gaussian & paley & hadamard\\
\hline
&\multicolumn{5}{|c|}{$m=24$, $k=3$}\\
\hline
train RMSE &0.805 & 0.791 & 0.783 & \textbf{0.780} & 0.782 \\
test RMSE &0.902 & 0.893 & 0.880 & \textbf{0.879} & 0.882 \\
runtime &2.60 & 3.22  & 3.98 & 3.49 & 3.49 \\
\hline
&\multicolumn{5}{|c|}{$m=24$, $k=12$}\\
\hline
train RMSE &0.770 &  \textbf{0.764} & 0.767 & \textbf{0.764} & 0.765 \\
test RMSE &0.872 & 0.870 & \textbf{0.866} & 0.868 & 0.868 \\
runtime &4.24 &4.38 & 4.92 & 4.50 & 4.61  \\
\hline
\end{tabular}
\caption{Full results for Movielens 1-M, distributed over $m=24$ nodes total. Runtime is in hours. An uncoded scheme running full batch L-BFGS has a train/test RMSE of 0.757 / 0.862, and a runtime of 14.11 hours.}
\label{t-movielens24}
\end{table}

\section{Efficient encoding using Steiner ETF}\label{ap:steiner}
We first describe Steiner ETF, based on the construction proposed in \cite{FickusMixon_12}. 
\paragraph{Steiner equiangular tight frames.}

Let $v$ be a power of 2, let $H \in \mathbb{R}^{v \times v}$ be a real Hadamard matrix, and let $h_i$ be the $i$th column of $H$, for $i=1,\dots,v$. Consider the matrix $V \in \lbp 0,1\rbp^{v \times v(v-1)/2}$, where each column is the incidence vector of a distinct two-element subset of $\lbp 1,\dots,v\rbp$. For instance, for $v=4$,
\begin{align*}
V = \lb
\begin{array}{cccccc}
1&1&1&0&0&0\\
1&0&0&1&1&0\\
0&1&0&1&0&1\\
0&0&1&0&1&1
\end{array}
\rb.
\end{align*}
Note that each of the $v$ rows have exactly $v-1$ non-zero elements. We construct Steiner ETF $S$ as a $v^2 \times \frac{v(v-1)}{2}$ matrix by replacing each 1 in a row with a distinct column of $H$, and normalizing by $\sqrt{v-1}$. For instance, for the above example, we have
\begin{align*}
S = \frac{1}{\sqrt{3}} \lb 
\begin{array}{cccccc}
h_2&h_3&h_4&0&0&0\\
h_2&0&0&h_3&h_4&0\\
0&h_2&0&h_3&0&h_4\\
0&0&h_2&0&h_3&h_4
\end{array}
\rb.
\end{align*}
In general, this procedure results in a matrix $S$ with redundancy factor $\beta = 2\frac{v}{v-1}$. In full generality, Steiner ETFs can be constructed for larger redundancy levels; we refer the reader to \cite{FickusMixon_12} for a full discussion of these constructions. 

\paragraph{Efficient distributed encoding.}
Steiner ETF allows for a distributed and efficient implementation for encoding a given matrix $X$. Note that the encoding matrix $S$ consists of $v$ blocks, each corresponding to a row of $V$. Consider the following partition for $S$:
\begin{align*}
S = \lb\begin{array}{cccc} S_1^\top & S_2^\top & \dots & S_v^\top\end{array}\rb^\top,
\end{align*}
where $S_i \in \mathbb{R}^{v \times v(v-1)/2}$ is a horizontal block. Note that multiplication with a block $S_i X$ can be computed by directly finding the non-zero column indices for the block $S_i$, which is given by $B_{1,i} \cup B_{2,i}$, where
\begin{align*}
B_{1,i} &= \lbp\frac{1}{2}j(2v-1-j) + (i-1-j), \text{for $j=0,\dots,i-2$} \rbp\\
B_{2,i} &= \lbp \frac{1}{2}(i-1)(2v-1-i) + j, \text{for $j=1,\dots,v-i$}\rbp
\end{align*} 
Multiplication can then be implemented by simply taking a Hadamard transform of the corresponding rows of $X$, whose indices are given by $B_{1,i}\cup B_{2,i}$. In the case of dimension mismatch, one can append zero rows to $X$, or remove some rows from $S$ to make the multiplication well-defined.

Therefore, one can partition the blocks $\lbp S_i \rbp$ across worker nodes, and worker node $k$ can read the corresponding rows of $X$ from the pool of data, given by $\bigcup_{i in I_k} B_{1,i} \cup B_{2,i}$, where $I_k$ is the blocks assigned to worker $k$, and then apply Fast Hadamard Transform for each block. Note that processing of blocks can be further parallelized within each node, using multiple cores.

In practice, we have observed that the performance of Steiner ETF significantly improves if the rows of $SX$ are shuffled after encoding. This can be implemented by having the nodes exchange rows of $SX$ after encoding; however, this incurs a significant communication cost. A more practical approach could be one where each block is assigned to multiple nodes at random, \emph{i.e.}, each block is encoded by multiple nodes. The worker nodes can then drop a subset of their encoded rows such that each row is retained by exactly one node, based on some pre-defined row-allocation rule. Note that this has the same effect as having the nodes randomly exchange encoded rows with each other.

\paragraph{Space complexity.} One might raise the point that many large-scale datasets are sparse, and this sparsity is lost after encoding, significantly increasing the memory usage. To address this issue, first consider the case where each node has access to the entire dataset $(X,y)$, for the sake of argument. Then a worker node would compute the gradient corresponding to block $i$ using the order of operations represented by the following parenthesization:
\begin{align*}
g_i(w_t) = \lp X^\top \lp S_i^\top \lp S_i \lp Xw-y\rp\rp\rp\rp,
\end{align*}
since this would only require matrix-vector multiplications. Now, note that for a given block $S_i$, there are only $v-1$ non-zero rows, out of $v(v-1)/2$. Therefore, in order to compute $g_i (w_t)$ in the above order of operations, one would only need to store the rows of $X$ and $y$ corresponding to these non-zero rows, and then apply the transformation corresponding to $S_i$ whenever needed. If each node is assigned $\frac{v}{m}$ blocks, this would require storing at most $\frac{v(v-1)}{m}$ sparse rows, which means that the memory usage would only increase by a constant factor (on the order of redundancy factor).

\end{appendices}








\end{document}